\documentclass[paper]{elsarticle}
\usepackage{hyperref}
\usepackage{changepage}
\usepackage{mathrsfs}
\usepackage{amssymb}
\usepackage{verbatim}
\usepackage{amsmath}
\usepackage{amsthm}
\usepackage{tabularx}
\usepackage{arydshln}

\journal{X}

\newtheorem{theorem}{Theorem}

\begin{document}

\begin{frontmatter}

\title{On the Equity of Nuclear Norm Maximization in Unsupervised Domain Adaptation}

\author[address1,address2]{Wenju Zhang}
\author[address1,address2]{Xiang Zhang\corref{mycorrespondingauthor}}
\author[address3]{Qing Liao\corref{mycorrespondingauthor}}
\author[address1,address2]{Long Lan}
\author[address1,address2]{Mengzhu~Wang}
\author[address4]{Wei Wang}
\author[address2]{Baoyun Peng}
\author[address5]{Zhengming Ding}

\cortext[mycorrespondingauthor]{Corresponding author\\
$~~~~~$ \emph{Email address}: zhangxiang08@nudt.edu.cn (X. Zhang), liaoqing@hit.edu.cn (Q. Liao).
}
%\ead{zhangxiang08@nudt.edu.cn}

\address[address1]{Institute for Quantum Information $\&$ State Key Laboratory of High Performance Computing, National University of Defense Technology, Changsha, China}
\address[address2]{College of Computer, National University of Defense Technology, Changsha, China}
\address[address3]{Department of Computer Science and Technology, \\Harbin Institute of Technology, Shenzhen, China}
\address[address4]{DUT-RU International School of Information Science $\&$ Engineering, \\Dalian University of Technology, Dalian, China}
\address[address5]{Tulane University, New Orleans, LA, USA}

\begin{abstract}
Nuclear norm maximization has shown the power to enhance 
the transferability of unsupervised domain adaptation model (UDA) {in an empirical scheme}. 

In this paper, we identify a new property termed \textit{equity}, which indicates the balance degree of predicted classes, to demystify the efficacy of nuclear norm maximization for UDA theoretically. With this in mind, we offer a new discriminability-and-equity maximization paradigm built on squares loss, such that predictions are equalized explicitly. To verify its feasibility and flexibility, two new losses termed Class Weighted Squares Maximization (CWSM) and Normalized Squares Maximization (NSM), are proposed to maximize both predictive discriminability and equity, from the class level and the sample level, respectively. Importantly, we theoretically relate these two novel losses (\emph{i.e.}, CWSM and NSM) to the equity maximization under
mild conditions, and empirically suggest the importance of the predictive equity in UDA. Moreover, it is very efficient to realize the equity constraints in both losses. Experiments of cross-domain image classification on three popular benchmark datasets show that both CWSM and NSM contribute to outperforming the corresponding counterparts.
\end{abstract}


\end{frontmatter}

% \linenumbers

\section{Introduction}

Deep learning has accomplished great success and dramatically improved the state-of-the-arts in various applications, especially for supervised learning tasks. The exciting performance of deep learning is believed to rely on a large deal of labeled training data. However, massive data acquisition and labeling process is always time-consuming and expensive. Instead of labeling new data, one alternative solution is to utilize existing labeled data or learned knowledge to facilitate different but related new tasks. A problem arising from this scenario is that there exists distribution discrepancy between existing data and new data, which is also known as domain shift. The existence of domain shift violates the basic assumption for standard machine learning, i.e., training data and test data are drawn from the same distribution, so that it leads to unsatisfied performance. Domain adaptation is one sort of techniques which mitigate domain shift in order to transfer specific knowledge from one domain (source domain) to another domain (target domain).

Known as a special case of transfer learning, domain adaptation has been extensively studied, and a considerable body of methods has been proposed in the past two decades \cite{pan2009survey,shao2014transfer,wang2018deep}.
According to whether the target domain has labels or not, domain adaptation can be grouped into supervised domain adaptation and unsupervised domain adaptation (UDA). Here, we focus on UDA which is more challenging and practical.
Early efforts in this regard mainly belong to non-deep machine learning methods, including instance-based methods \cite{gong2013connecting,gretton2009covariate,huang2007correcting} which learn weights of instances to reduce the cross-domain distribution discrepancy, and feature-based methods \cite{gong2012geodesic,long2013transfer,pan2010domain} which learn a common subspace where the domain shift is reduced. Afterwards, feature-based methods have become the mainstream direction in the field of UDA \cite{long2013transfer, gong2016domain, zhang2017joint, zhang2020maximum}.

With the advance of deep learning, 
the learned features by deep neural networks have shown a certain capacity of transferability \cite{donahue2014decaf,glorot2011domain,oquab2014learning}. However, this transferability is still weak, when source domain and target domain have distant distance \cite{yosinski2014transferable}. Besides, many theoretical studies on domain adaptation have also been explored \cite{ben2010theory,liu2017understanding,zhang2019bridging}.

 More recently, deep UDA methods which unify deep feature learning and domain adaptation have attracted much attention and significantly pushed the boundaries of UDA.

Among them, semi-supervised learning (SSL) based methods \cite{chen2019domain,cui2020towards,grandvalet2005semi} become one of the hot research topics due to the akin relation between SSL and UDA in problem definition. As a classical SSL method, entropy minimization \cite{grandvalet2005semi} has been widely incorporated into UDA models \cite{long2015learning,xu2019larger} to enhance predictive discriminability. Recently, maximum squares loss \cite{chen2019domain} has been proposed to tackle the over large gradient problem for easy-to-transfer samples which the entropy minimization confronts. %In a more recent study
Differently, batch nuclear norm maximization (BNM) \cite{cui2020towards} tries to maximize the nuclear norm of mini-batch target outputs to enhance both predictive discriminability and diversity. BNM outperforms both entropy minimization and maximum squares loss by a large margin, indicating the importance of diversity for UDA.

Inspired by the success of BNM in UDA where the nuclear norm serves as the common proxy of both the Frobenius-norm and the matrix rank, we continue to delve into the effectiveness of the nuclear norm from a different perspective. Towards this goal, a new property \textit{equity} is introduced to demystify the new working mechanism of nuclear norm theoretically. Different from the diversity which recognizes the individual differences of target outputs without caring about the class-size fairness, the \emph{equity} emphasizes the total class-size balance degree within a set of target outputs. Note that the nuclear norm yields the equity implicitly rather than explicitly. This is not obvious, and could make it hard to understand this property and limits its applicability. Thus, from the new perspective of the equity, we devise two novel loss functions for UDA by explicitly incorporating the equity constraints into squares loss, that is, the class weighted squares maximization (CWSM) and normalized squares maximization (NSM), to show the importance of the equity in UDA.

Particularly, 
 in CWSM, we re-weight the class-wise squares loss by the weight that is inversely proportional to the soft class size. In theory, the optimal solution to both CWSM and BNM is equivalent under mild conditions, which indirectly uncovers the mechanism of the equity truly working in BNM again. Besides,
 in NSM, we divide the squares loss by a normalization term which can be rewritten as the square-sum of soft class sizes in a special case.
 We theoretically prove that the optimal solution to CWSM and NSM has maximal discriminability and equity under mild conditions. Experiments on three benchmark datasets including Office-31, Office-Home, and VisDA-2017 show that both CWSM and NSM outperform several state-of-the-art methods, which also implies the feasibility of our claim.

This paper is an extension to our previous work \cite{zhang2020robust}. Compared with \cite{zhang2020robust}, several substantial differences have been made
as follows: 1) We propose a new loss function termed CWSM as an alternative implementation of discriminability and equity maximization. 2) We theoretically prove that the optimal solution to CWSM has maximal discriminability and equity. 3) We generalize NSM by introducing a parameter $r$ and interpret the robust NSM in our previous work as a special case of NSM with $r=0.5$. 4) We add considerable analysis, discussions, and experiments in this paper. In a nutshell, we make four major contributions as follows:

\begin{itemize}
  \item We are the first to introduce a new property the equity to demystify the new working mechanism of the nuclear norm in theory.
  \item On the basis of the equity, two novel loss functions including class weighted squares maximization (CWSM) and normalized squares maximization (NSM), are then employed to simultaneously encourage prediction discriminability and equity.
  \item We theoretically prove that the optimal solution to CWSM and NSM has maximal discriminability and equity under mild conditions.
  \item A series of experiments are conducted on three benchmark datasets, and such results verify the efficacy of both CWSM and NSM as compared to state-of-the-art methods.
\end{itemize}

\section{Related work}

Most deep networks based unsupervised domain adaptation (UDA) methods can be categorized into four groups: 1) distribution matching based methods \cite{tzeng2014deep,long2015learning,sun2016deep,long2016unsupervised,long2017deep} which explicitly reduce the cross-domain distribution discrepancy by minimizing a certain distribution metric, such as widely used maximum mean discrepancy \cite{gretton2012kernel}. 2) adversarial learning based methods \cite{ganin2016domain, long2018conditional,zhang2019domain} which reduce domain shift by learning cross-domain embeddings that can not be distinguished by the domain discriminator. 3) semi-supervised learning (SSL) based methods \cite{chen2019domain,cui2020towards,grandvalet2005semi} which implicitly cope with domain shift using SSL techniques. and 4) other methods \cite{xu2019larger,kang2019contrastive}. Our study focuses on SSL based UDA methods.

Various SSL techniques including self-training, self-ensembling, and entropy minimization have been adopted to build UDA models for cross-domain segmentation or classification tasks \cite{zou2018unsupervised,zou2019confidence,french2018self,perone2019unsupervised,vu2019advent,chen2019domain}. Entropy minimization \cite{grandvalet2005semi} is one of the most popular SSL methods which reduces prediction uncertainty by minimizing Shannon Entropy of prediction. There exist UDA models \cite{long2015learning,xu2019larger} that use entropy minimization to enhance predictive discriminability so as to promote classification performance. However, entropy minimization faces the problem that the gradient of the entropy is biased towards easy-to-transfer samples.
Recently, the maximum squares loss \cite{chen2019domain} has been proposed to improve entropy minimization by lowering the gradient of easy-to-transfer samples. In a more recent study \cite{cui2020towards}, the nuclear norm of mini-batch prediction matrix is maximized to enhance both prediction discriminability and diversity. Since taking diversity into consideration, nuclear norm maximization outperforms entropy minimization and maximum squares. However, the working mechanism of nuclear norm maximization remains ambiguous. In addition, based on the similar motivation of balancing classes, there are methods \cite{zou2018unsupervised,zou2019confidence,chen2019domain} which use imbalanced learning \cite{he2009learning} to deal with the class imbalance problem in UDA. However, the pseudo-labels are needed for these approaches because imbalanced learning is supervised but there is lack of the labels of target domain in UDA.

\section{The Proposed Algorithm}

\subsection{Algorithm Overview}
We begin with notations and problem formulation. Denoting the labeled data in source domain as ${D_s} = \left\{ {\left( {{x^s},{y^s}} \right)} \right\}$ and unlabeled data in target domain as ${D_t} = \{ {{x^t}} \}$, unsupervised domain adaptation (UDA) attempts to train a classifier which is well performed on $D_t$ by utilizing $D_s$. A general objective for UDA models focused in this paper can be formulated as:
\begin{equation}
\label{equ:obj}
    \mathop {\min }\limits_\theta  {{\mathbb{E}}_{( {x_i^s,{y^s}} ) \sim {D_s}}}{L_{ce}}( {p_i^s,{y_s}} ) + \lambda{{\mathbb{ E}}_{x_j^t \sim {D_t}}}{L_{t}}( {p_j^t}),
\end{equation}
where $\theta$ represents network parameters, $p_i^s$ ($p_j^t$) represents the class probability prediction by feeding $x_i^s$ ($x_j^t$) into the deep network, ${L_{ce}}$ represents the cross entropy loss, ${L_{t}}$ represents the loss function for target prediction, and $\lambda$ is the weight parameter for ${L_{t}}$. By substituting ${L_{t}}$ with the maximum squares loss \cite{chen2019domain}, or the nuclear norm maximization loss \cite{cui2020towards}, or the proposed class weighted squares maximization loss, or the proposed normalized squares maximization loss, the corresponding UDA model can be obtained.

Using mini-batch gradient descent as the optimization algorithm, we denote the probability prediction matrix of target samples in a mini-batch as $P \in \mathbb R^{B \times C}$ , where $B$ represents the mini-batch size and $C$ represents the number of classes. According to the property of probability, there is ${P_{ic}} \ge 0$ for $1\le i \le B, 1\le j\le C$ and $\sum_{c = 1}^C {{P_{ic}}}  = 1$ for $1\le i\le B$. We also denote the $i$-th row of $P$ as $P_i$. For matrix $P\in \mathbb R^{B\times C}$, the nuclear norm is defined as $\|P\|_* = Tr(\sqrt{P^TP}) = \sum_{i=1}^{\min{(B,C)}}\sigma_i$, where $\sigma_i$ is the $i$-th singular value of $P$. %\ZD{is the $i$-th singular value}

Since our work is highly related to maximum squares (MS) loss \cite{chen2019domain} and batch nuclear norm maximization (BNM) loss \cite{cui2020towards}, we revisit their formulations from the work \cite{chen2019domain,cui2020towards} here for future use. The MS loss is defined as,
\begin{equation}
    \label{equ:ms}
    MS(P) = - \frac{1}{B}\sum\limits_{i = 1}^B {\sum\limits_{c = 1}^C {{{\left( {{P_{ic}}} \right)}^2}} },
\end{equation}

and the BNM loss is defined as,
\begin{equation}
    \label{equ:bnm}
    BNM(P) = - \frac{1}{B}{\|P\|_* }.
\end{equation}

Both MS and BNM losses are simple but show the power to enhance transferability of UDA. In contrast to MS, BNM exhibits strong discriminability but has deficient optimization process. Instead, MS has efficient implementation. But they still have some theoretical limits on the power. Towards this end, this section will introduce a new property termed \emph{equity}, which lays the underlying foundation about the efficacy of nuclear norm maximization theoretically.

\subsection{Equity: A New Property of Nuclear Norm Maximization}
\begin{table}
\centering
\caption{
The numbers of categories, matrix ranks, and nuclear norms of three examples of prediction matrix.
\label{tab:diversity_example}}
\begin{tabular}{cccc}
\hline
Prediction&\#Category&Matrix rank&Nuclear norm\\
\hline
$P_1$&1&1&2\\
$P_2$&2&2&2.73\\
$P_3$&2&2&2.82\\
\hline
\end{tabular}
\end{table}
In the study \cite{cui2020towards}, the diversity is a key concept which can be measured by the number of predicted categories and approximated by the matrix rank. Since matrix rank is non-convex and hard to optimize, its convex envelope nuclear norm is used instead. However, the behaviors of the nuclear norm not only contain the diversity and discriminability but also embrace the new working mechanism, i.e., the \textit{equity}. \newtheorem{myDef}{Definition}
For convenience, this paper defines a specific equity constraint for the purpose of guiding us to design two types of the equity constraints in the subsequent sections.
\begin{myDef}
({\bf \emph{The equity}}) A classification model is called the \emph{equity}, if its corresponding target prediction $P\in R^{B\times C}$ meets the following condition:
\begin{align}
n_p=n_q, \forall p\ne q\\
\sum\limits_{c=1}^C n_c =B\label{definition}
\end{align}
where $n_p=\sum\limits_{i=1}^{B} P_{ip}$ is the number of samples belong to the class $p$, and $B$ and $C$ denote batch size %\ZD{the number of samples $\rightarrow$ batch size?}
and classes, respectively.
\end{myDef}

Note that \eqref{definition} always holds due to the property of the probability matrix $P$, thus it can be removed from the definition above. For clarity, we keep it there. Obviously, the equity indicates that the number of each predicted class output by a classifier is (near) identical.

Literally, the diversity merely recognizes the inclusion of different classes without caring about the class-size fairness, while the equity additionally emphasizes the total class-size balance degree. To illustrative this difference, we define three matrices as follows,
\begin{equation}
{P_1} = \left[ {\begin{array}{*{2}{c}}
1&0\\
1&0\\
1&0\\
1&0\\
\end{array}} \right],
{P_2} = \left[ {\begin{array}{*{2}{c}}
1&0\\
1&0\\
1&0\\
0&1\\
\end{array}} \right],
{P_3} = \left[ {\begin{array}{*{2}{c}}
1&0\\
1&0\\
0&1\\
0&1\\
\end{array}} \right],
\end{equation}
and compute their numbers of categories, matrix ranks, and nuclear norms in Table \ref{tab:diversity_example}. Note that although nuclear norm is considered to encode both discriminability and diversity, since the discriminability of $P_1$, $P_2$, and $P_3$ are identical to each other, nuclear norm here only measures diversity.
It can be seen that, in terms of the number of categories and matrix rank, $P_2$ and $P_3$ have the same diversity, and they are more diverse than $P_1$. This conforms the literal meaning of diversity and the original intention of usage of diversity. Nevertheless, if using nuclear norm to measure the diversity, the diversity of $P_3$ should equal to that of $P_2$. In reality, the nuclear norm of $P_3$ is higher than that of $P_2$. Obviously, the difference between the nuclear norms of both $P_2$ and $P_3$ could indicates the other property in addition to predictive diversity. By comparing their norms, the predicted classes of $P_3$ is more balanced than that of $P_2$. That is, in this case nuclear norm is most likely to measure how \textit{balanced} the predicted classes are, \emph{i.e.}, \textit{equity}.

In empirical studies, we find that maximizing nuclear norm of predictions always enforces the predicted classes to be balanced. To support this claim, we offer the following theory analysis. {\bf Theorem \ref{theo:nuclear_norm_cws}} shows that the matrix $P$ with one-hot rows and balanced class sizes is an optimal solution to maximizing $\|P\|_*$. This confirms that it is more proper to interpret the working mechanism of nuclear norm maximization as the equity maximization.

\begin{theorem}
\label{theo:nuclear_norm_cws}
% (\textbf{The equity})
There exists one optimal solution to maximizing $\|P\|_*$ with one-hot rows and balanced class sizes.
\end{theorem}
\begin{proof}

Since nuclear norm is convex and the feasible region of $P$ is a convex and compact set, according to  Theorem 7.42 in \cite{beck2014introduction}, at least one optimal solution to maximizing $\|P\|_*$ is attained at extreme points which are matrices with one-hot rows.

Next we search the optimal solution from these extreme points.
Assuming that the rows of matrix $P\in \mathbb R^{B\times C}$ are one-hot, it can be verified that
\begin{equation}
P^{T}P=\Sigma\Sigma,
\end{equation}
where $\Sigma\in \mathbb R^{C\times C}$ is a diagonal matrix with diagonal entries as $\Sigma_{cc}=\sqrt{n_c}$, wherein ${n_c} = \sum_{i = 1}^B {{P_{ic}}} $.
According to the definition of nuclear norm, there is,
\begin{equation}
\label{equ:nuclear_norm_nc}
\|P\|_*=\textit{Tr}\left(\sqrt{P^TP}\right)=\textit{Tr}\left(\Sigma\right)=\sum_{c=1}^C\sqrt{n_c}.
\end{equation}

Then, the optimization problem maximizing $\|P\|_*$ can be cast as:
\begin{equation}
\label{equ:nuclear_norm_one_hot}
\mathop {\max }\limits_{{n_c} \in \mathbb{N}_0} \sum_{c = 1}^C \sqrt{n_c} \;\;\;s.t.\sum_{c = 1}^C {{n_c}}  = B,
\end{equation}
where $\mathbb N_0$ represents non-negative integers. The optimal solution to \eqref{equ:nuclear_norm_one_hot} can be obtained by $n_c^* \in \{ {\lfloor {B/C} \rfloor ,\lceil {B/C} \rceil } \}$, and the number of $\left\lfloor {B/C} \right\rfloor $ denoted by $m$ is given by solving
\begin{equation}
\label{equ:compute_m}
m\left\lfloor {B/C} \right\rfloor  + (C - m)\left\lceil {B/C} \right\rceil  = B.
\end{equation}

This result shows that the optimal solution to maximizing $\|P\|_*$ has balanced class sizes.

\end{proof}

\subsection{A Discriminability-and-Equity Maximization Paradigm for UDA}
Inspired by the above-mentioned analysis about the working mechanism of nuclear norm maximization, we propose a discriminability-and-equity maximization paradigm to solve UDA and explore other implementations to verify its effectiveness.
In the following, we will give a possible explanation why discriminability and equity maximization works.

Recall that previous SSL based UDA methods, such as entropy minimization \cite{grandvalet2005semi} and maximum squares \cite{chen2019domain}, only encourage prediction discriminability but do not enforce any constraint with respect to the equity. Thus, the predicted classes could form any distribution which may differ from the ground truth. The difference between distributions could hurt the model performance on target domain. By enforcing an appropriate equity constraint, the predictions could have the opportunity to become more close to ground truth so as to promote adaptation performance.

Based on the above-mentioned analysis, we will design two novel loss functions with explicit equity constraints in the frame of discriminability-and-equity maximization paradigm for UDA in the following sections. In this paper we primarily focus on verify the
feasibility of the equity and its efficacy for UDA in order to explicitly observe the significance of the proposed alternative equity constraints for UDA.
This could help motivate the practitioners to design more sound equity constraints.

Towards this goal, we realize the equity constraint from the viewpoint of class and sample levels, respectively. From class level, we expect the contribution of each class to the model to be equivalent whatever the class is large or small. A simple solution is reweighing the classes by dividing the sample number of each class (see Section \emph{Class weighted squares maximization}). In view of the sample level, the equity constraint enforces the number of all the samples into the same class to be uniform with the class size of each sample (see Section \emph{Normalized Squares Maximization}). The details are described as below.

\subsection{Class Weighted Squares Maximization}
We choose the maximum squares loss \eqref{equ:ms} \cite{chen2019domain} as the base loss to achieve discriminability maximization due to its simplicity and effectiveness.
% due to its concise form and decent mathematical property.
To realize the equity constraint, one intuitive solution is to introduce the weights to reweigh class size based on the idea that the small classes have large weights while the large classes have small weights.
% Usually, the weight for a class is set to the inverse of class size.
Accordingly, we define the class weighted squares (CWS) as follows:
\begin{equation}
\label{equ:cws}
CWS\left( P \right) = \frac{1}{C}\sum\limits_{c = 1}^C {\frac{{\sum\limits_{i = 1}^B {{{( {{P_{ic}}} )}^2}} }}{(\sum\limits_{i = 1}^B {{P_{ic}}})^r}},
\end{equation}
where $r$ is a positive parameter to be discussed in the following. According to Equ. \eqref{equ:cws}, the \textbf{class weighted squares maximization} (CWSM) can be defined as:
\begin{equation}
\label{equ:cwsm}
\textit{CWSM}(P) = -\textit{CWS}(P).
\end{equation}

In \eqref{equ:cws}, since $P_{ic}$ represents the probability of the sample $i$ belonging to class $c$, $\sum_{i = 1}^B {{P_{ic}}}$ can be viewed as the soft class size of the $c$-th class. Compared with the squares loss in \eqref{equ:ms}, the introduced weight $\frac{1}{(\sum_{i=1}^BP_{ic})^r}$ for each class is inversely proportional to the class size. Hence, maximizing CWS leads to discriminative and balanced prediction results. Actually, this point has been verified by the theoretical analysis given by \textbf{Theorem \ref{theo:CWSM}} and \textbf{Theorem \ref{theo:cwsm_bal}}. \textbf{Theorem \ref{theo:CWSM}} and \textbf{Theorem \ref{theo:cwsm_bal}} respectively show that the maximum of CWS is attained at the feasible solution with maximal discriminability and equity.

The parameter $r$ in \eqref{equ:cws} is to control the influence of the equity constraint on the whole loss function, i.e., balance discriminability and equity. The larger the $r$ is, the more intensive the equity constraint is. Here, we empirically restrict $0\le r \le  1$. When $r=0$, CWS becomes the squares loss with the only difference of a constant coefficient. In this case, there is no any constraint on balancing classes. When $r=1$, the class weight is the inverse of class size, which has been widely used to tackle the class imbalance problem in supervised learning.
%For $0<r<1$, the extent of enforcing balanced classes ranges from 0 to the maximum.
A moderate $r$ means relatively balanced trade-off between discriminability and equity.
A case study and the experimental analysis on the effect of the parameter $r$ will be shown in Sec. \ref{sec:case_study} and Sec. \ref{sec:par_ana}, respectively.

%By incorporating the class weight which is  into the squares loss,  maximizing CWS encourage both discriminate and class balanced predictions.

Note that a similar formulation (Equ. (13) in \cite{chen2019domain}) to ours has been proposed to deal with the class imbalance problem. Nevertheless, the proposed CWSM differs from the approach of \cite{chen2019domain} in three aspects.
Firstly, Equ. (13) in \cite{chen2019domain} computes the class frequency using the pseudo labels, but CWSM use the soft probability prediction to compute the soft class size. Soft probabilities tend to contain more information and are less sensitive to noises than hard pseudo labels. Secondly, CWSM unifies discriminability and equity maximization in a overall learning process. In contrast, the method in \cite{chen2019domain} consists of two separate steps: obtaining pseudo labels and re-weighting classes based on class frequencies, where the
first step is non-differentiable. From the perspective of methodology, the  approach in \cite{chen2019domain} restricts its idea in the scope of the  class imbalance problem in supervised learning. In contrast, our method does not need pseudo labels and are more elegant for unsupervised learning scenario. At last, the theoretical analysis about our method provides solid theoretical foundations.

\begin{theorem}
\label{theo:CWSM}
(\textbf{The discriminability}) When $0<r<1$, one necessary condition for the optimal solution $P^*$ to CWSM is that each row of $P^*$ is one-hot.
\end{theorem}
\begin{proof}
For simplicity, we omit the constant coefficient $\frac{1}{C}$ in CWS, because the constant does not affect the final conclusion. In the following, we will prove this theorem using a proof by contradiction. Accordingly, suppose that $P'$ is an optimal solution to CWSM, and its $k$-th row ${P'_k}$ is not one-hot.

Let $x$ be a column vector with length $C$, and define a function $f(x)$ associated with $P'$ as:
\begin{equation}
\label{equ:fx}
f(x)=CWS(P'|_{k\rightarrow x}),
\end{equation}
where $P'|_{k\rightarrow x}$ represents the matrix by substituting the $k$-row of $P'$ with $x^T$. To make $P'|_{k\rightarrow x}$ valid, the definition domain of $f$ should meet $\{x\in\mathbb R^{C}\colon x_c\ge 0 \text{ for } 1\le c \le C  \text{ and } \sum_{c=1}^C x_c=1\}$, which is obviously convex, according to the definition of the convex set. Based on \eqref{equ:cws}, \eqref{equ:fx} can be rewritten as,
\begin{equation}
\label{equ:cws_pk}
f(x) = \sum\limits_{c = 1}^C {\frac{{\sum\limits_{i = 1,i \ne k}^B {{{\left( {{P'_{ic}}} \right)}^2}}  + {{\left( {{x_{c}}} \right)}^2}}}{({{\sum\limits_{i = 1,i \ne k}^B {{P'_{ic}}}  + {x_{c}}}})^r}}.
\end{equation}

Letting $a_c=\sum_{i = 1,i \ne k}^B {{{P'}_{ic}}}$, and $b_c=\sum_{i = 1,i \ne k}^B {{{\left( {{{P'}_{ic}}} \right)}^2}}$, \eqref{equ:cws_pk} can be rewritten as:
\begin{equation}
f(x) = \sum\limits_{c = 1}^C {\frac{{{b_c} + x_c^2}}{{({{a_c} + {x_c}})^r }}}.
\end{equation}

By derivation, the Hessian matrix of $f(x)$ denoted by $H\in \mathbb R^{C\times C}$ is a diagonal matrix and its diagonal entries $H_{cc}$ for $1\le c \le C$ are,
\begin{equation}
    %{H_{cc}} = \frac{{8a_c^2 + 3x_c^2 + 8{a_c}{x_c} + 3{b_c}}}{{4{{\left( {{a_c} + {x_c}} \right)}^{\frac{5}{2}}}}}\\
    {H_{cc}} = \frac{{(1-r)(2-r)x_c^2 + 4(1-r)a_cx_c + 2a_c^2 + r(1+r)b_c}}{{{{\left( {{a_c} + {x_c}} \right)}^{r+2}}}}.
\end{equation}

Recall that $a_c\ge 0$, $b_c\ge 0$, and $0 \le x_c \le 1$, we will discuss the range of $H_{cc}$ in two cases: $a_c > 0$, and $a_c = 0$.

1) For $a_c> 0$, it is obvious that $H_{cc}> 0$.

2) For $a_c = 0$, by the notation, all the $P'_{ic}$ with $i\ne k$ equal zeros, which imply that $b_c = 0$. Then, there is,
\begin{equation}
H_{cc} = \frac{(1-r)(2-r)}{x_c^r} > 0.
\end{equation}

In light of both cases 1) and 2), there holds $H_{cc}> 0$. Hence, $H$ is positive definite, thus $f(x)$ is strictly convex.

Denote the standard orthogonal basis of $\mathbb   R^C$ as $\{e_1, ..., e_C\}$, where the $c$-th entry of $e_c$ is 1, and the other entries are 0s. For concise, let $v=P'_k$. Then, $v$ can be rewritten as,
\begin{equation}
v=\sum\limits_{c=1}^C{v_ce_c}.
\end{equation}

Let $z$ represent the number of non-zero elements of $v$, resort $v$ such that the first $z$ elements are greater than 0 and the last $C-z$ elements are equal to 0, and define the map from the new index $i$ of the reordered vector to the original index as $\pi _i$, i.e.,
$v_{\pi_i}>0$ for $1\le i \le z$, and $v_{\pi_i}=0$ for $z+1\le i \le C$. Obviously, there is,
\begin{equation}
v=\sum\limits_{c=1}^C{v_{\pi_c}e_{\pi_c}}.
\end{equation}

Since $v$ is not one-hot, there is $z\ge 2$ and $0<v_{\pi_i}<1$ for $1\le i \le z$. Recalling that $f(x)$ is strictly convex, there is,
\begin{equation}
\label{equ:cov_max_1}
\begin{array}{l}
\;\;\;\;f({v}) \\= f( {{v_{{\pi _1}}}{e_{{\pi _1}}} + (1 - {v_{{\pi _1}}})\sum\limits_{c = 2}^C {\frac{{{v_{{\pi _c}}}{e_{{\pi _c}}}}}{{1 - {v_{{\pi _1}}}}}} } )\\
 < {v_{{\pi _1}}}f({e_{{\pi _1}}}) + (1 - {v_{{\pi _1}}})f(\sum\limits_{c = 2}^C {\frac{{{v_{{\pi _c}}}{e_{{\pi _c}}}}}{{1 - {v_{{\pi _1}}}}}} )\\
 \le \max \{ {f({e_{{\pi _1}}}),f(\sum\limits_{c = 2}^C {\frac{{{v_{{\pi _c}}}{e_{{\pi _c}}}}}{{1 - {v_{{\pi _1}}}}}} )} \},
\end{array}
\end{equation}
where the first equality is based on the definition of strictly convex function, and ${\sum_{c = 2}^C {\frac{{{v_{{\pi _c}}}{e_{{\pi _c}}}}}{{1 - {v_{{\pi _1}}}}}} }$ is in the feasible definition region because $\sum_{c = 2}^C {\frac{{{v_{{\pi _c}}}}}{{1 - {v_{{\pi _1}}}}}}  = 1$. Moreover, for $1<i<z$, there is,
\begin{equation}
\label{equ:cov_max_i}
\begin{array}{l}
\;\;\;\;f(\sum\limits_{c = i}^C {\frac{{{v_{{\pi _c}}}{e_{{\pi _c}}}}}{{1 - \sum\limits_{j = 1}^{i - 1} {{v_{{\pi _i}}}} }}} ) \\
= f( {\frac{{{v_{{\pi _i}}}{e_{{\pi _i}}}}}{{1 - \sum\limits_{j = 1}^{i - 1} {{v_{{\pi _j}}}} }} + ( {1 - \frac{{{v_{{\pi _i}}}}}{{1 - \sum\limits_{j = 1}^{i - 1} {{v_{{\pi _j}}}} }}} )\sum\limits_{c = i + 1}^C {\frac{{{v_{{\pi _c}}}{e_{{\pi _c}}}}}{{1 - \sum\limits_{j = 1}^i {{v_{{\pi _j}}}} }}} } )\\
 < \frac{{{v_{{\pi _i}}}}}{{1 - \sum\limits_{j = 1}^{i - 1} {{v_{{\pi _j}}}} }}f({e_{{\pi _i}}}) + (1 - \frac{{{v_{{\pi _i}}}}}{{1 - \sum\limits_{j = 1}^{i - 1} {{v_{{\pi _j}}}} }})f(\sum\limits_{c = i + 1}^C {\frac{{{v_{{\pi _c}}}{e_{{\pi _c}}}}}{{1 - \sum\limits_{j = 1}^i {{v_{{\pi _j}}}} }}} )\\
 \le \max \{ f({e_{{\pi _i}}}),f(\sum\limits_{c = i + 1}^C {\frac{{{v_{{\pi _c}}}{e_{{\pi _c}}}}}{{1 - \sum\limits_{j = 1}^i {{v_{{\pi _j}}}} }}} )\},
\end{array}
\end{equation}
where $0<{\frac{{{v_{{\pi _i}}}}}{{1 - \sum_{j = 1}^{i - 1} {{v_{{\pi _j}}}} }}}<1$, and ${\sum_{c = i + 1}^C {\frac{{{v_{{\pi _c}}}}}{{1 - \sum_{j = 1}^i {{v_{{\pi _j}}}} }}} }=1$. According to the previous definition of $z$, there is,
\begin{equation}
\label{equ:cov_max_z}
f(\sum\limits_{c = z}^C {\frac{{{v_{{\pi _c}}}{e_{{\pi _c}}}}}{{1 - \sum\limits_{j = 1}^{z - 1} {{v_{{\pi _j}}}} }}} ) = f( {{e_{{\pi _z}}}}).
\end{equation}

By combining \eqref{equ:cov_max_1}, \eqref{equ:cov_max_i}, and \eqref{equ:cov_max_z}, there holds,
\begin{equation}
\label{equ:pk_max}
f({v}) < \max \{ f( {{e_{{\pi _1}}}}),...,f( {{e_{{\pi _z}}}})\}.
\end{equation}

Letting
\begin{equation}
c^* = \arg\max_{1\le c \le z} f( {{e_{{\pi _c}}}}),
\end{equation}
it can be used to construct a new matrix $P^*=P'|_{k\rightarrow e_{\pi _{c^*}}}$, i.e., substituting the $k$-row of $P'$ with $e_{\pi _{c^*}}^T$. According to \eqref{equ:pk_max}, there is $CWS( {{P^*}}) > CWS( {P'})$. This contradicts the supposition that $P'$ is an optimal solution.
%shows that there exists a solution whose $k$-row is one-hot and its $CWS$ value is larger than that of $P'_k$.
\end{proof}

\begin{theorem}
\label{theo:cwsm_bal}
(\textbf{The equity}) When $0<r<1$, the optimal solution to CWSM has balanced class sizes.
\end{theorem}
\begin{proof}
Based on {\bf Theorem \ref{theo:CWSM}}, we search the optimal solution from the space within which every matrix consists of one-hot rows. Accordingly, \eqref{equ:cws} can be rewritten as:
\begin{equation}
\label{eq:cws_one_hot}
\textit{CWS}(P) = \frac{1}{C}\sum_{c=1}^Cn_c^{1-r},
\end{equation}
where ${n_c} = \sum_{i = 1}^B {{P_{ic}}} $. Then, the CWSM problem can be recast as:
\begin{equation}
\label{equ:cws_one_hot}
\mathop {\max }\limits_{{n_c} \in \mathbb{N}_0} \sum_{c = 1}^C {n_c^{1-r}} \;\;\;s.t.\sum_{c = 1}^C {{n_c}}  = B,
\end{equation}
where $\mathbb{N}_0$ stands for non-negative integers. Since $0<1-r<1$, the optimal solution $n_c^*$ to \eqref{equ:cws_one_hot} can be obtained by $n_c^* \in \{ {\lfloor {B/C} \rfloor ,\lceil {B/C} \rceil } \}$, and the number of $\left\lfloor {B/C} \right\rfloor $ denoted by $m$ is given by solving \eqref{equ:compute_m}.
This indicates that the optimal solution has balanced class sizes.
\end{proof}

\subsection{Normalized Squares Maximization}
%Since the maximum squares loss only encourages discriminate prediction for each target sample and it does not enforce any constraint on class sizes, the final classification result may suffer from the class size imbalance problem.

In this section, we adopt another strategy to enforce the equity constraint. Different from CWSM which introduces class-balanced weights, we here consider to divide the square loss by a normalization term, and this indicates minimizing the normalization term leads to balanced classes. At first, we define the normalized squares (NS) as follows:
\begin{equation}
\label{equ:ns}
{\textstyle
%\resizebox{0.92\hsize}{!}{
\textit{NS}(P) =\frac{{\sum\limits_{i = 1}^B {\sum\limits_{c = 1}^C {{{( {{P_{ic}}} )}^2}} } }}{{\sum\limits_{i,j = 1,i \ne j}^B ({\sum\limits_{c = 1}^C {{P_{ic}}{P_{jc}}} })^r  + \alpha \sum\limits_{i = 1}^B {\sum\limits_{c = 1}^C {{{\left( {{P_{ic}}} \right)}^2}} } }} + \epsilon \sum\limits_{i = 1}^B {\sum\limits_{c = 1}^C {{{\left( {{P_{ic}}} \right)}^2}} },
}
%}
\end{equation}
where $\alpha \ge 0$ is a parameter to avoid too large gradient near the optimal solution, $\epsilon$ is a small number for the convenience of theoretical analysis, and $0\le r \le 1$ is a parameter to tune the extent of the equity constraint. $\epsilon$ can be assigned by the rule that $\epsilon = 0$ for $B < C$ and $\epsilon = 10^{-6}$ for $B \ge C$. Then, the \textbf{normalized squares maximization} (NSM) loss can be formulated as:
\begin{equation}
\textit{NSM}(P) =  -\textit{NS}(P).
\end{equation}

Like $r$ in \eqref{equ:cws}, $r$ in \eqref{equ:ns} has the similar effect on controlling the extent of equity constraint, or balancing discriminability and equity. When $r=0$, \eqref{equ:ns} becomes
\begin{equation}
\label{equ:ns_c1}
\textit{NS}(P) = \frac{1}{\alpha}(1-\frac{B^2-B}{B^2-B+\alpha {\sum\limits_{i = 1}^B {\sum\limits_{c = 1}^C {{{\left( {{P_{ic}}} \right)}^2}} } } }) + \epsilon {\sum\limits_{i = 1}^B {\sum\limits_{c = 1}^C {{{\left( {{P_{ic}}} \right)}^2}} } }.
\end{equation}

In this case, NSM has the similar optimization behaviors with MaxSquare because there is $\textit{NSM}(P_1)>\textit{NSM}(P_2)$  if $\textit{MS}(P_1)>\textit{MS}(P_2)$. This means that there is no any constraint on equity when $r=0$.

%We next analyse NSM under the situation of $r=1$.
When $r=1$,
\eqref{equ:ns} can be rewritten as follows for better demonstration:
\begin{equation}\label{equ:ns_essence}
\textit{NS}(P) = \frac{1}{{\frac{{\sum\limits_{c = 1}^C {{( {\sum\limits_{i = 1}^B {{P_{ic}}} } )^2}} }}{{\sum\limits_{i = 1}^B {\sum\limits_{c = 1}^C {{{\left( {{P_{ic}}} \right)}^2}} } }} + \alpha  - 1}} + \varepsilon \sum\limits_{i = 1}^B {\sum\limits_{c = 1}^C {{{\left( {{P_{ic}}} \right)}^2}} }.
\end{equation}

Based on \eqref{equ:ns_essence}, maximizing normalized squares has two-fold effects: maximizing $\sum_{i = 1}^B {\sum_{c = 1}^C {{{( {{P_{ic}}})}^2}} }$ and minimizing $\sum_{c = 1}^C {{{( {\sum_{i = 1}^B {{P_{ic}}} } )}^2}}$. Maximizing $\sum_{i = 1}^B {\sum_{c = 1}^C {{{( {{P_{ic}}})}^2}} }$ encourages predictive discriminability as shown in \cite{chen2019domain}. Considering that $\sum_{i = 1}^B {{P_{ic}}}$ represents the soft class size of the $c$-th class, minimizing $\sum_{c = 1}^C {{{( {\sum_{i = 1}^B {{P_{ic}}} } )}^2}}$ enforces the class sizes to be balanced so as to encourage predictive equity. Accordingly, we
qualitatively conclude that NSM enlarges discriminability and equity when $r=1$. In fact, the conclusion about the discriminability and equity can be guaranteed by \textbf{Theorem \ref{theo:1}} and \textbf{Theorem \ref{theo:2}}.
Moreover, when $r=1$, $\alpha=1$ and $\epsilon=0$, \eqref{equ:ns_essence} becomes:
\begin{equation}
NS(P) = \frac{{\sum\limits_{i = 1}^B {\sum\limits_{c = 1}^C {{{\left( {{P_{ic}}} \right)}^2}} } }}{{\sum\limits_{c = 1}^C {{{(\sum\limits_{i = 1}^B {{P_{ic}}} )^2}}} }},
\end{equation}
where the denominator is actually the sum of squared class sizes. This special case helps us recognize the mechanism of NSM and our motivation.

When $0<r<1$, the intensity of the equity constraint increases as $r$ rises. The empirical study of a special case for the influence of $r$ is given in Sec. \ref{sec:case_study}.
The theoretical analysis of NSM given by \textbf{Theorem \ref{theo:3}} shows that NSM simultaneously encourages the discriminability and equity under the situation of $0<r<1$ and $B \le C$.
%We believe it is also true for $B > C$ and we remain the theoretical proof open for the further work.

\textbf{Probability Explanation.} Here we attempt to interpret NSM from the probability perspective. Considering the fact that $P_{ic}$ represents the probability of that sample $i$ belongs to class $c$, $\sum_{c=1}^C P_{ic}P_{jc}$ represents the probability of that sample $i$ and sample $j$ belong the same class. Moreover, $\frac{1}{B}\sum_{i=1}^B\sum_{c=1}^C P_{ic}^2$ represents the probability of that any sample belongs to a single class and $\frac{1}{B^2-B}\sum_{i,j=1,i\ne j}^B\sum_{c=1}^C P_{ic}P_{jc}$ represents the probability of that any two different sample belongs to the same class. Accordingly, maximizing $\sum_{i=1}^B\sum_{c=1}^C P_{ic}^2$ urges one sample to belong to only one class, i.e., encourage discriminability; minimizing $\sum_{i,j=1,i\ne j}^B\sum_{c=1}^C P_{ic}P_{jc}$ urges different samples to belong to different classes as soon as possible, i.e., encourage equity. The parameter $r$ re-scales the probability. According to the character of the function $y=x^r$ for $0\le x\le1$ and $0\le r\le1$, the smaller the $r$ is,
the larger the normalization term is, so the weaker the equity constraint is.

\begin{theorem}
\label{theo:1}
(\textbf{The discriminability}) When $r=1$, one necessary condition for the optimal solution $P^*$ to NSM is that each row of $P^*$ is one-hot.
\end{theorem}
\begin{proof}
In the following, we will prove {\bf Theorem \ref{theo:1}} using a proof by contradiction.

Suppose $P'$ is an optimal solution to NSM, and its $k$-th row ${P'_k}$ is not one-hot vector. We construct a new matrix ${P^{\rm{*}}}$ by replacing ${P'_k}$  in ${P'}$  with a new vector $P_k^*$ whose element is defined as:
\begin{equation}
P_{kc}^* = \left\{ \begin{array}{l}
1,\;if\;c = \arg \mathop {\min }\limits_d \sum_{i = 1,i \ne k}^B {{{P'}_{id}}} \\
0,\;otherwise
\end{array} \right..
\end{equation}

Next, we prove that $NS\left( {{P^*}} \right) > NS\left( {P'} \right)$ .

First, according to the supposition that ${P'_k}$ is not one-hot vector, there is,
\begin{equation}
\begin{array}{l}
\sum_{c = 1}^C {{{( {P_{kc}^*} )}^2}}  = {( {\sum_{c = 1}^C {{{P'}_{kc}}} })^2} > \sum_{c = 1}^C {{{( {{{P'}_{kc}}} )}^2}}
\end{array},
\end{equation}
where the inequality holds because
\begin{equation}
\begin{array}{l}
{( {\sum_{c = 1}^C {{{P'}_{kc}}} })^2}
 = \sum_{c = 1}^C {{{( {{{P'}_{kc}}})}^2}}  + \sum_{i,j = 1,i \ne j}^C {{{P'}_{ki}}{{P'}_{kj}}}
 \end{array},
\end{equation}
wherein $\sum_{i ,j = 1,i \ne j}^C {{{P'}_{ki}}{{P'}_{kj}}}  > 0$.
Thus, there holds,
\begin{equation}\label{equ:squares}
\begin{array}{l}
\;\;\;\;\sum\limits_{i = 1}^B {\sum\limits_{c = 1}^C {{{\left( {P_{ic}^*} \right)}^2}} }  \\= \sum\limits_{c = 1}^C {{{( {P_{kc}^*} )}^2}}  + \sum\limits_{i = 1,i \ne k}^B {\sum\limits_{c = 1}^C {{{( {{{P'}_{ic}}} )}^2}} } \\
 > \sum\limits_{c = 1}^C {{{( {{{P'}_{kc}}} )}^2} + \sum\limits_{i = 1,i \ne k}^B {\sum\limits_{c = 1}^C {{{( {{{P'}_{ic}}} )}^2}} } } \\
= \sum\limits_{i = 1}^B {\sum\limits_{c = 1}^C {{{( {{{P'}_{ic}}} )}^2}} }
\end{array}.
\end{equation}
%\begin{equation}\label{equ:squares}
%\textstyle
%\sum_{i = 1}^B {\sum_{c = 1}^C {{{\left( {P_{ic}^*} \right)}^2}} }  > \sum_{i = 1}^B {\sum_{c = 1}^C {{{\left( {{{P'}_{ic}}} \right)}^2}} }
%\end{equation}

Second, there is,
\begin{equation}
\begin{array}{l}
\;\;\;\;\sum\limits_{c = 1}^C { {P_{kc}^*\sum\limits_{i = 1,i \ne k}^B {{{P'}_{ic}}} } } \\ = \sum\limits_{i = 1,i \ne k}^B {{{P'}_{i{c^*}}}} \\
 = \sum\limits_{c = 1}^C { {{{P'}_{kc}}\sum\limits_{i = 1,i \ne k}^B {{{P'}_{i{c^*}}}} }}  \\
\le \sum\limits_{c = 1}^C { {{{P'}_{kc}}\sum\limits_{i = 1,i \ne k}^B {{{P'}_{ic}}} } ,}
 \end{array}
\end{equation}
where ${c^*} = \arg \mathop {\min }\limits_d \sum_{i = 1,i \ne k}^B {{{P'}_{id}}} $.
%\begin{equation}
%\textstyle
%\sum_{c = 1}^C {( {P_{kc}^*\sum_{i = 1,i \ne k}^B {{{P'}_{ic}}} } )}  \le \sum_{c = 1}^C {( {{{P'}_{kc}}\sum_{i = 1,i \ne k}^B {{{P'}_{ic}}} } )}
%\end{equation}
Thus, there holds,
%\begin{equation}\label{equ:pair_product}
%\textstyle
%\sum\limits_{i,j = 1,i \ne j}^B {\sum\limits_{c = 1}^C {P_{ic}^*P_{jc}^*} }
% \le \sum\limits_{i,j = 1,i \ne j}^B {\sum\limits_{c = 1}^C {{{P'}_{ic}}{{P'}_{jc}}} }
%\end{equation}
\begin{equation}\label{equ:pair_product}
\begin{array}{l}
\;\;\;\;\sum\limits_{i,j = 1,i \ne j}^B {\sum\limits_{c = 1}^C {P_{ic}^*P_{jc}^*} } \\
 = 2\sum\limits_{c = 1}^C { {P_{kc}^*\sum\limits_{i = 1,i \ne k}^B {{{P'}_{ic}}} }}  + \sum\limits_{i,j = 1,i \ne j\ne k}^B {\sum\limits_{c = 1}^C {{{P'}_{ic}}{{P'}_{jc}}} } \\
 \le 2\sum\limits_{c = 1}^C { {{{P'}_{kc}}\sum\limits_{i = 1,i \ne k}^B {{{P'}_{ic}}} }}  + \sum\limits_{i,j = 1,i \ne j \ne k}^B {\sum\limits_{c = 1}^C {{{P'}_{ic}}{{P'}_{jc}}} } \\
 = \sum\limits_{i,j = 1,i \ne j}^B {\sum\limits_{c = 1}^C {{{P'}_{ic}}{{P'}_{jc}}} }
\end{array}.
\end{equation}

Now, we rewrite \eqref{equ:ns_essence} as:
\begin{equation}\label{equ:ns_ratio}
\textit{NS}\left( P \right) = \frac{1}{{\frac{{\sum\limits_{i,j = 1,i \ne j}^B {\sum\limits_{c = 1}^C {{P_{ic}}{P_{jc}}} } }}{{\sum\limits_{i = 1}^B {\sum\limits_{c = 1}^C {{{( {{P_{ic}}})}^2}} } }} + \alpha } }  + \epsilon{\sum\limits_{i = 1}^B {\sum\limits_{c = 1}^C {{{( {{P_{ic}}})}^2}} } },
\end{equation}
%where $\sum_{i = 1}^B {\sum_{c = 1}^C {{{\left( {{P_{ic}}} \right)}^2}} }  > 0$ and thus the transformation is valid.

Obviously, there is,
\begin{equation}\label{equ:pair_product_range}
\left\{ \begin{array}{l}
\sum_{i,j = 1,i \ne j}^B {\sum_{c = 1}^C {{P_{ic}}{P_{jc}}} }  > 0\;if\;B > C\\
\sum_{i,j = 1,i \ne j}^B {\sum_{c = 1}^C {{P_{ic}}{P_{jc}}} }  \ge 0\;if\;B \le C
\end{array} \right..
\end{equation}

Recall the assignment of $\epsilon$
\begin{equation} \label{equ:eps}
\left\{ \begin{array}{l}
\epsilon  \ge 0\;if\;B > C\\
\epsilon  > 0\;if\;B \le C
\end{array} \right..
\end{equation}
By combining \eqref{equ:squares}, \eqref{equ:pair_product}, \eqref{equ:pair_product_range},
\eqref{equ:eps}, and
\eqref{equ:ns_ratio}, it is easy to obtain that $NS\left( {{P^*}} \right) > NS\left( {P'} \right)$. This contradicts the supposition that $P'$ is the optimal solution.
\end{proof}

\begin{theorem}
\label{theo:2}
(\textbf{The equity}) When $r=1$, the optimal solution to NSM has balanced class sizes.
\end{theorem}
\begin{proof}
Based on {\bf Theorem \ref{theo:1}}, we search the optimal solution from the space within which every matrix consists of one-hot rows. Accordingly, \eqref{equ:ns_essence} can be rewritten as:
\begin{equation}
\textit{NS}(P) = \frac{B}{{\sum_{c = 1}^C {n_c^2}  + \left( {\alpha  - 1} \right)B}} + \epsilon B,
\end{equation}
where ${n_c} = \sum_{i = 1}^B {{P_{ic}}} $. Then, the NSM problem can be cast as:
\begin{equation}
\textstyle
\label{equ:class_size_balance}
\mathop {\min }\limits_{{n_c} \in \mathbb{N}_0} \sum_{c = 1}^C {n_c^2} \;\;\;s.t.\sum_{c = 1}^C {{n_c}}  = B,
\end{equation}
where $\mathbb{N}_0$ stands for non-negative integers. The optimal solution $n_c^*$ to \eqref{equ:class_size_balance} can be obtained by $n_c^* \in \{ {\lfloor {B/C} \rfloor ,\lceil {B/C} \rceil } \}$, and the number of $\left\lfloor {B/C} \right\rfloor $ denoted by $m$ is given by solving \eqref{equ:compute_m}.
This indicates that the optimal solution has balanced class sizes.
\end{proof}

\begin{theorem}
\label{theo:3}
(\textbf{The generic equity}) When $B\le C$ and $0 < r < 1$, a sufficient and necessary condition for the optimal solution ${P^*}$ to NSM is that each row of ${P^*}$ is one-hot, and for each column of ${P^*}$, there is at most one entry is 1.
\end{theorem}
\begin{proof}
We first prove the sufficiency. Obviously, there is
%$RNS(P)  \le \frac{1}{\alpha } + \epsilon B$.
\begin{equation}
{\textstyle
%\scriptstyle
\textit{NS}(P) = \frac{1}{{\frac{{\sum\limits_{i,j = 1,i \ne j}^B {({\sum\limits_{c = 1}^C {{P_{ic}}{P_{jc}}} })^r } }}{{\sum\limits_{i = 1}^B {\sum\limits_{c = 1}^C {{{\left( {{P_{ic}}} \right)}^2}} } }} + \alpha }} + \epsilon \sum\limits_{i = 1}^B {\sum\limits_{c = 1}^C {{{\left( {{P_{ic}}} \right)}^2}} } \le \frac{1}{\alpha } + \epsilon B
}.
\end{equation}

It can be verified that $\textit{NS}( {{P^*}}) = \frac{1}{\alpha } + \epsilon B$. Thus, $P^*$ is an optimal solution.
%It is easy to prove the necessity using proof by contradiction. We omit it due to the limitation of space.
Next we prove the necessity using proof by contradiction.
Assuming an optimal solution $P^*$ does not meet the condition, there is  $\textit{NS}({{P^*}}) < \frac{1}{\alpha } + \epsilon B$, which contradicts the definition of optimal solution.
\end{proof}

\subsection{Case Study}

\begin{figure}%[!htb]
	\centering
	\includegraphics[scale=0.4]{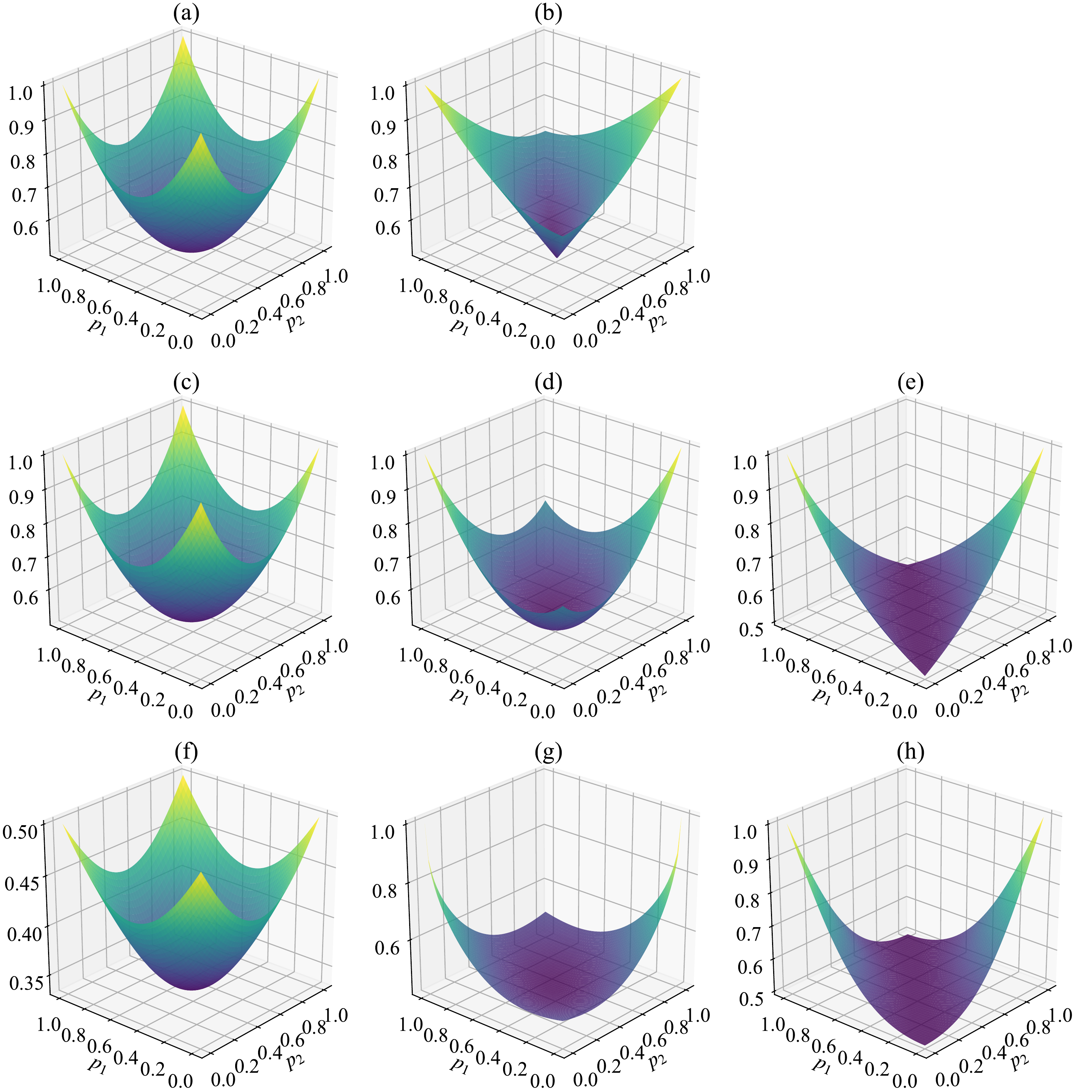}
	\caption{Visualization of the negative loss functions of (a) MS, (b) BNM, (c) CWSM ($r=0$), (d) CWSM ($r=0.5$), (e) CWSM ($r=1$), (f) NSM ($r=0$), (g) NSM ($r=0.5$), and (h) NSM ($r=1$) in the case of $B=2$ and $C=2$.}
    \label{fig:obj}
\end{figure}

\label{sec:case_study}
For better understanding, in this section, we will study the proposed two loss functions, CWSM and NSM, compared with MS and BNM in a special case, $B=2$ and $C=2$. In this case, $P$ has the the following form,
\begin{equation}
P = \left[ {\begin{array}{*{20}{c}}
{{p_1}}&{1 - {p_1}}\\
{{p_2}}&{1 - {p_2}}
\end{array}} \right],
\end{equation}
where $0\le p_1 \le 1$ and $0\le p_2 \le 1$. For simplicity, denote the four extreme points of the feasible region as follows,
\begin{equation}
\begin{array}{l}
{P_1} = \left[ {\begin{array}{*{20}{c}}
0&1\\
0&1
\end{array}} \right],{P_2} = \left[ {\begin{array}{*{20}{c}}
1&0\\
0&1
\end{array}} \right],\vspace{0.2cm}\\
{P_3} = \left[ {\begin{array}{*{20}{c}}
0&1\\
1&0
\end{array}} \right],{P_4} = \left[ {\begin{array}{*{20}{c}}
1&0\\
1&0
\end{array}} \right].
\end{array}
\end{equation}

Fig. \ref{fig:obj} shows the surfaces of different loss functions. Note that a negative sign is added to all losses for better demonstration and thus the optimization problem becomes maximizing the objective. From Fig. \ref{fig:obj}, three observations can be made:

1) It can be seen that $P_1$, $P_2$, $P_3$, and $P_4$ are optimal solutions to MS, while $P_2$ and $P_3$ are optimal solutions to BNM, CWSM ($r>0$), and NSM ($r>0$).
%Different from MaxSquare, the BNM, CWSM, and NSM values of $P_2$ and $P_3$ are higher than those of $P_1$ and $P_4$.
Since the two classes of $P_2$ and $P_3$ are more balanced that those of $P_1$ and $P_4$, BNM, CWSM ($r>0$), and NSM ($r>0$) favor to class balanced predictions.

2) The objective surfaces of CWSM ($r=0$) and NSM ($r=0$) are identical and similar to that of MS, respectively. This verifies that there is no equity constraint for CWSM and NSM when $r=0$.  Fig. \ref{fig:obj} (c)-(e) and (f)-(h) illustrate the effect of $r$ on the proposed two loss functions. For both CWSM and NSM, larger $r$ leads to more expectation on balanced solutions.
% The phenomenon can be explained that there is a trade-off between discriminability and balanceability in CWSM and NSM and the trade-off is tuned by the parameter $r$.

%When $r=0.5$, $P_1$ and $P_4$ are local optimal solutions to CWSM and NSM. Intuitively, it is reasonable because the true class distribution is unknown and imbalanced classes are possible to be truth. In this situation, $P_1$ and $P_4$ could be better choice since they are more discriminate.

3) By comparing BNM, CWSM ($r=0.5$), and NSM ($r=0.5$), the BNM objective values and gradients of solutions which are less discriminate (near $p_1=p_2=0.5$) are relatively large but those values are relatively small for NSM ($r=0.5$).
The uncertain solutions are more likely to be outliers and noises. Thus, NSM ($r=0.5$) can be considered robust to outliers and noises.

\subsection{Computation Complexity}

\begin{table}%[!htb]
\label{tab:computation_complexity}
\centering
\caption{Computation complexities of compared methods.} %Best accuracy valu e of each column is highlighted in bold.}
\label{tab:coomplexity}
\begin{tabular}{cc}
\hline
Method&Complexity\\
\hline
MaxSquare&$O(BC)$\\
{BNM}&$\min\{O(B^2C), O(BC^2)\}$\\
{CWSM}&$O(BC)$\\
{NSM ($r=1$)}&$O(BC)$\\
{NSM ($r\ne 1$)}&$O(B^2C)$\\
\hline
\end{tabular}
\end{table}

Based on \eqref{equ:cws}, \eqref{equ:ns_essence}, and \eqref{equ:ns}, we list the computation complexities of CWSM, NSM and two compared methods including MaxSquare and BNM in Table \ref{tab:coomplexity}. It can be seen that the computation complexity of CWSM is lower than those of both BNM and NSM. Since the mini-batch size $B$ is usually small in the cross-domain image classification task, the computation complexity of NSM is comparable to that of BNM.

% to understand its working mechanism
\subsection{Connections between BNM, CWSM, and NSM}
\begin{table}
\centering
\caption{CWS and NS values of three example prediction matrices. (Table \ref{tab:diversity_example} continued) \label{tab:loss_example}}
\begin{tabular}{cc:cc}
\hline
Prediction&Nuclear norm&CWS value&NS value\\
\hline
$P_1$&2&1&0.25\\
$P_2$&2.73&1.37&0.4\\
$P_3$&2.82&1.41&0.5\\
\hline
\end{tabular}
\end{table}
As mentioned previously, CWSM and NSM are designed for maximizing the discriminability and equity, and BNM actually does the same thing. We append the CWS and NS values to Table \ref{tab:diversity_example} and summarize them in Table \ref{tab:loss_example}. From Table \ref{tab:loss_example}, for CWS and NS, there also is $l(P_1)<l(P_2)<l(P_3)$, where $l(\cdot)$ represents the corresponding loss function. This gives an intuitive example that CWSM and NSM encourage the equity. Interestingly, CWS has a close connection with nuclear norm on the extreme points: assume that the rows of matrix $P\in \mathbb R^{B\times C}$ are one-hot, there is,
\begin{equation}
\|P\|_*=C\cdot CWS(P)|_{r=0.5},
\end{equation}
which can be easily obtained according to \eqref{equ:nuclear_norm_nc} and \eqref{eq:cws_one_hot}.

In summary, BNM, CWSM, and NSM can be considered to reach a same goal in different ways. Despite this, it needs to be noted that CWSM is more computationally efficient than BNM, and CWSM and NSM have a parameter $r$ which is able to dynamically control discriminability and equity. Besides, we uncover the working mechanism of BNM and succeed in designing two new loss functions according to the working mechanism: the discriminability and equity maximization. The efficacy of CWSM and NSM verify this point in next section.

\begin{figure*}%[!htb]
	\centering
	\begin{adjustwidth}{-3.1cm}{}
	\includegraphics[scale=0.58]{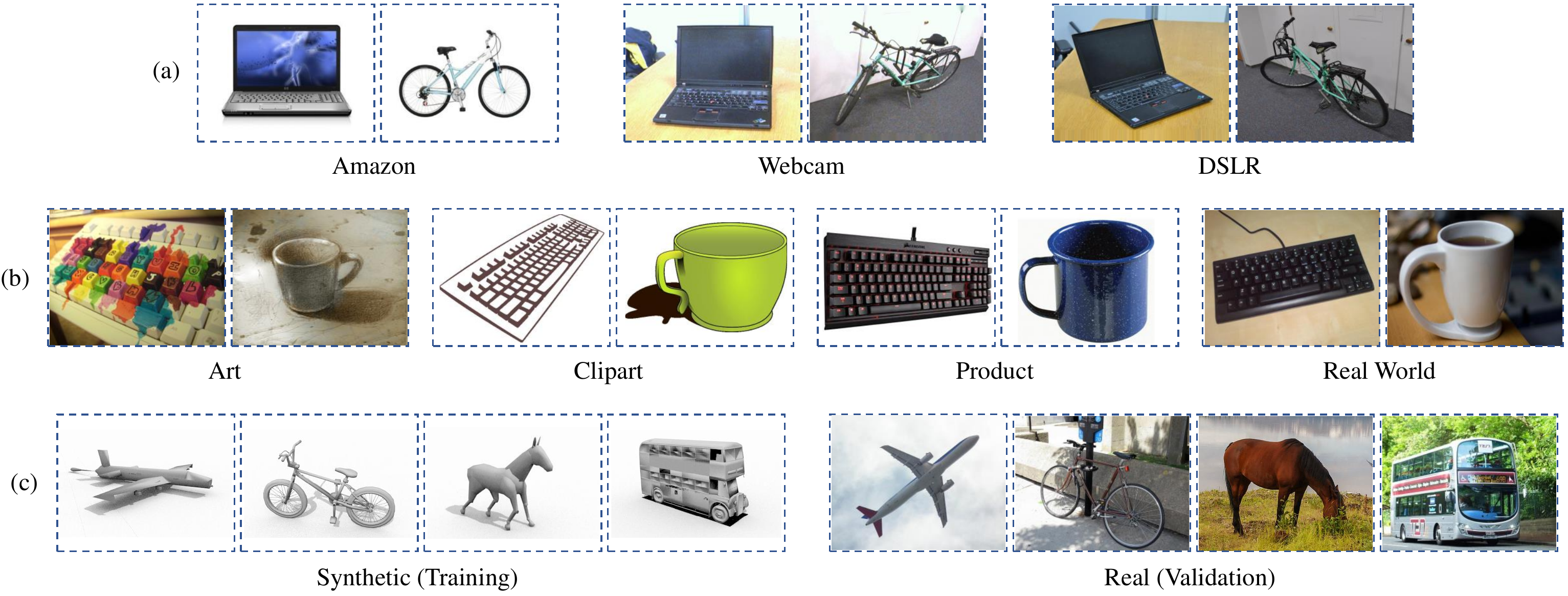}
	\caption{Sample images from (a) office-31, (b) Office-Home, and (c) VisDA-2017 datasets.}
	\label{fig:dataset}	
	\end{adjustwidth}

\end{figure*}

\section{Experiments}
\label{sec:exp}

\begin{table}%[!htb]
% \small
%\normalsize
%\vspace{-0.4cm}
\centering
\caption{Accuracy (\%) of compared methods on the Office-31 dataset based on ResNet-50.} %Best accuracy value of each column is highlighted in bold.}
%\vspace{-0.5\baselineskip}

\label{tab:office}
\begin{tabular}{cccccccc} \hline
Method&{A$\to$W}&{D$\to$W}&{W$\to$D}&{A$\to$D}&{D$\to$A}&{W$\to$A}&{Avg}
\\ \hline
ResNet-50 \cite{he2016deep}                      &68.4&96.7&99.3&68.9&62.5&60.7&76.1\\
DAN \cite{long2015learning}                       &80.5&97.1&99.6&78.6&63.6&62.8&80.4\\
DANN \cite{ganin2016domain}                   &82.0&96.9&99.1&79.7&68.2&67.4&82.2\\
JAN \cite{long2017deep}                             &85.4&97.4&\underline{99.8}&84.7&68.6&70.0&84.3\\
%MCD&88.6&98.5&100.0&92.2&69.5&69.7&86.5\\
%GTA \cite{sankaranarayanan2018generate}&89.5 &97.9&99.8&87.7&72.8&71.4&86.5\\
CDAN \cite{long2018conditional}               &94.1&98.6&\bf100.0&92.9&71.0&69.3&87.7\\
SAFN \cite{xu2019larger}                            &90.1&98.6&\underline{99.8}&90.7&73.0&70.2&87.1\\
SymNets  \cite{zhang2019domain}&90.8&98.8&\bf 100.0&\bf93.9&74.6&72.5&88.4\\
MaxSquare \cite{chen2019domain}            &92.4&\bf99.1&\bf100.0&90.0&68.1&64.2&85.6\\
BNM* \cite{cui2020towards}                        &\underline{95.2}&98.6&99.6&92.3&\underline{75.1}&75.4&89.4\\
\hline
CWSM &94.5 &\underline{99.0} &99.7 &92.6  &\bf75.6 &\underline{75.9}&\underline{89.6}\\
% NSM &93.3&97.7&99.2&88.2&74.9&75.1&88.1\\
NSM &\bf95.4&98.8&99.5&\underline{93.8}&\bf75.6&\bf76.3&\bf89.9\\
\hline
\end{tabular}
%\vspace{-3mm}

\end{table}

\begin{table*}%[!htb]
% \small
\centering
\setlength{\tabcolsep}{0.9mm}
\caption{Accuracy (\%) of compared methods on the Office-Home dataset based on ResNet-50.} %Best accuracy value of each column is highlighted in bold.}
 \begin{adjustwidth}{-3.9cm}{}
\label{tab:office_home}
% \resizebox{\textwidth}{!}
{
\begin{tabular}{cccccccccccccc}
\hline
Method          & Ar$\to$Cl & Ar$\to$Pr & Ar$\to$Rw & Cl$\to$Ar & Cl$\to$Pr & Cl$\to$Rw & Pr$\to$Ar & Pr$\to$Cl & Pr$\to$Rw & Rw$\to$Ar & Rw$\to$Cl & Rw$\to$Pr & Avg  \\
\hline
ResNet-50 \cite{he2016deep} & 34.9  & 50.0  & 58.0  & 37.4  & 41.9  & 46.2  & 38.5  & 31.2  & 60.4  & 53.9  & 41.2  & 59.9  & 46.1 \\
DAN \cite{long2015learning} & 43.6  & 57.0  & 67.9  & 45.8  & 56.5  & 60.4  & 44.0  & 43.6  & 67.7  & 63.1  & 51.5  & 74.3  & 56.3 \\
DANN \cite{ganin2016domain}  &45.6& 59.3 &70.1 &47.0 &58.5 &60.9 &46.1 &43.7 &68.5 &63.2 &51.8& 76.8 &57.6 \\
JAN \cite{long2017deep}&45.9 &61.2 &68.9& 50.4 &59.7& 61.0& 45.8& 43.4 &70.3 &63.9 &52.4& 76.8& 58.3\\
CDAN \cite{long2018conditional} & 50.7  & 70.6  & 76.0  & 57.6  & 70.0  & 70.0  & 57.4  & 50.9  & 77.3  & 70.9  & 56.7  & 81.6 & 65.8  \\
%MCD \citesaito2018maximum & 48.9  & 68.3  & 74.6  & 61.3  & 67.6  & 68.8  & 57  & 47.1  & 75.1  & 69.1  & 52.2  & 79.6 & 64.1  \\
SAFN \cite{xu2019larger} & 52.0  & 71.7  & 76.3  & 64.2  & 69.9  & 71.9  & \underline{63.7}  & 51.4  & 77.1  & 70.9  & \underline{57.1}  & 81.5 & 67.3  \\
SymNets \cite{zhang2019domain} & 47.7  & 72.9  & 78.5  & 64.2  & 71.3  & 74.2  &\bf 64.2  & 48.8  & 79.5  & \bf74.5  & 52.6  & 82.7 & 67.6  \\
MaxSquare* \cite{chen2019domain} & 50.3  & 70.9  & 76.5  & 60.5  & 68.3  & 69.8  & 59.0  & 47.0  & 76.5  & 70.4  & 53.0  & 80.9  & 65.3 \\
%MDD \cite{zhang2019bridging}& 54.9  & 73.7  & 77.8  & 60.0  & 71.4  & 71.8  & 61.2  & 53.6  & 78.1  & 72.5  & 60.2  & 82.3 & 68.1  \\
BNM* \cite{cui2020towards}       & 54.3  & 74.9  & 79.3  & 64.2  & \underline{74.6}  & 74.6  & 61.8  & 52.7  & 80.2  & 71.5  & 56.4 & 83.0  & 69.0 \\
\hline
CWSM &\underline{55.8} &\underline{75.9} &\underline{79.8}  &\underline{65.6}  &\underline{74.6} &\underline{75.6} &63.3 &\underline{53.8} &\underline{80.7} &72.9 &\underline{57.1} &\underline{83.4} &\underline{69.9}\\
%NSM       & 53.9  & 72.5  & 77.9  & 61.6  & 72.3  & 73.4  & 61.3  & 51.4  & 77.8  & 68.9  & 56.1  & 81.2  & 67.4 \\
NSM      & \bf56.8  & \bf76.5  & \bf80.3  & \bf66.6  &\bf 75.6  &\bf 75.9  & {63.6}  & \bf53.9  & \bf81.0  & \underline{73.3}  & \bf57.9  & \bf83.8  &\bf 70.4\\
\hline
\end{tabular}
}
%\vspace{-1mm}
\end{adjustwidth}
\end{table*}

\begin{table*} %[!htb]
%\small
\centering
\caption{Accuracy (\%) of compared methods on the VisDA-2017 dataset based on ResNet-101.} %Best accuracy value of each column is highlighted in bold.}
%\vspace{-0.1\baselineskip}
\label{tab:visda}
\begin{adjustwidth}{-3.4cm}{}
\begin{tabular}{cccccccccccccc}
\hline
Method          & plane & bcycl & bus & car & horse & knife & mcycl & person & plant & sktbrd & train & truck & Avg  \\
\hline
ResNet-101 \cite{he2016deep} & 55.1  & 53.3  & 61.9  & 59.1  & 80.6  & 17.9  & 79.7  & 31.2  & 81.0  & 26.5  & 73.5  & 8.5  & 52.4 \\
DANN \cite{ganin2016domain}  &81.9& 77.7 &82.8 &44.3 &81.2 &29.5 &65.1 &28.6 &51.9 &54.6 &82.8& 7.8 &57.4 \\
DAN \cite{long2015learning} & 87.1  & 63.0  & 76.5  & 42.0  & 90.3  & 42.9  & 85.9  & 53.1  & 49.7  & 36.3  & 85.8  & 20.7  & 61.1 \\
MaxSquare* \cite{chen2019domain} & 92.4  & 42.4  & 81.2  & 74.8  & 89.2  & \bf89.7  & \bf92.4  & 55.8  & 80.5  & 46.1  & 86.6  & 15.4  & 70.5 \\
CDAN \cite{long2018conditional} & 85.2  & 66.9  & 83.0  & 50.8  & 84.2  & 74.9  & 88.1  & 74.5  & 83.4  & 76.0  & 81.9  & 38.0 & 73.9  \\
%MCD \citesaito2018maximum & 48.9  & 68.3  & 74.6  & 61.3  & 67.6  & 68.8  & 57  & 47.1  & 75.1  & 69.1  & 52.2  & 79.6 & 64.1  \\
SAFN \cite{xu2019larger} & 93.6  & 61.3  & \bf84.1  & \bf70.6  & \bf94.1  & 79.0  & \underline{91.8}  & \bf79.6  & \underline{89.9}  & 55.6  & {89.0}  & 24.4 & 76.1  \\
BNM* \cite{cui2020towards}       &  \bf 94.0  & \bf82.7  & 77.2  & 65.2  & {93.1}  & 76.2  & 85.9  & 77.4  & 88.5  & 68.2  & \underline{89.3} & \bf47.9  & 78.8 \\
\hline
CWSM &93.2 &\underline{81.9} &\underline{83.2}  &\underline{69.9}  &{92.9} &{80.4} &86.8 &{76.8} &{89.7} &\bf79.5 &{88.3} &\underline{41.2} &\underline{80.3}\\
NSM      & \underline{93.8}  & 81.8  & 83.0  & 66.0  & \underline{93.2}  & \underline{81.8}  & {89.0}  & \underline{78.5}  & \bf92.1  & \underline{77.1}  & \bf89.7  & 38.3  & \bf80.4\\
\hline
\end{tabular}
\end{adjustwidth}
%\vspace{-1mm}
\end{table*}

%Absolute difference between class size distribution of predictions and uniform distribution on the Ar$\to$Cl task.
\begin{table}%[!htb]
\centering
\caption{Equity and discriminability of target prediction on the Ar$\to$Cl task.}
\label{tab:diff}
\begin{tabular}{ccc}
\hline
Method&Equity&Discriminability\\
\hline
%MaxSquare&0.391&\\
{CWSM ($r=0.5$)}&0.719&0.933\\
{CWSM ($r=1$)}&0.818&0.867\\
{NSM ($r=0.5$)}&0.774&0.945\\
{NSM ($r=1$)}&0.805&0.785\\
%{BNM}&0.757&0.918\\
\hline
\end{tabular}
\end{table}

\begin{figure}%[!htb]
	\centering
	\includegraphics[scale=0.3]{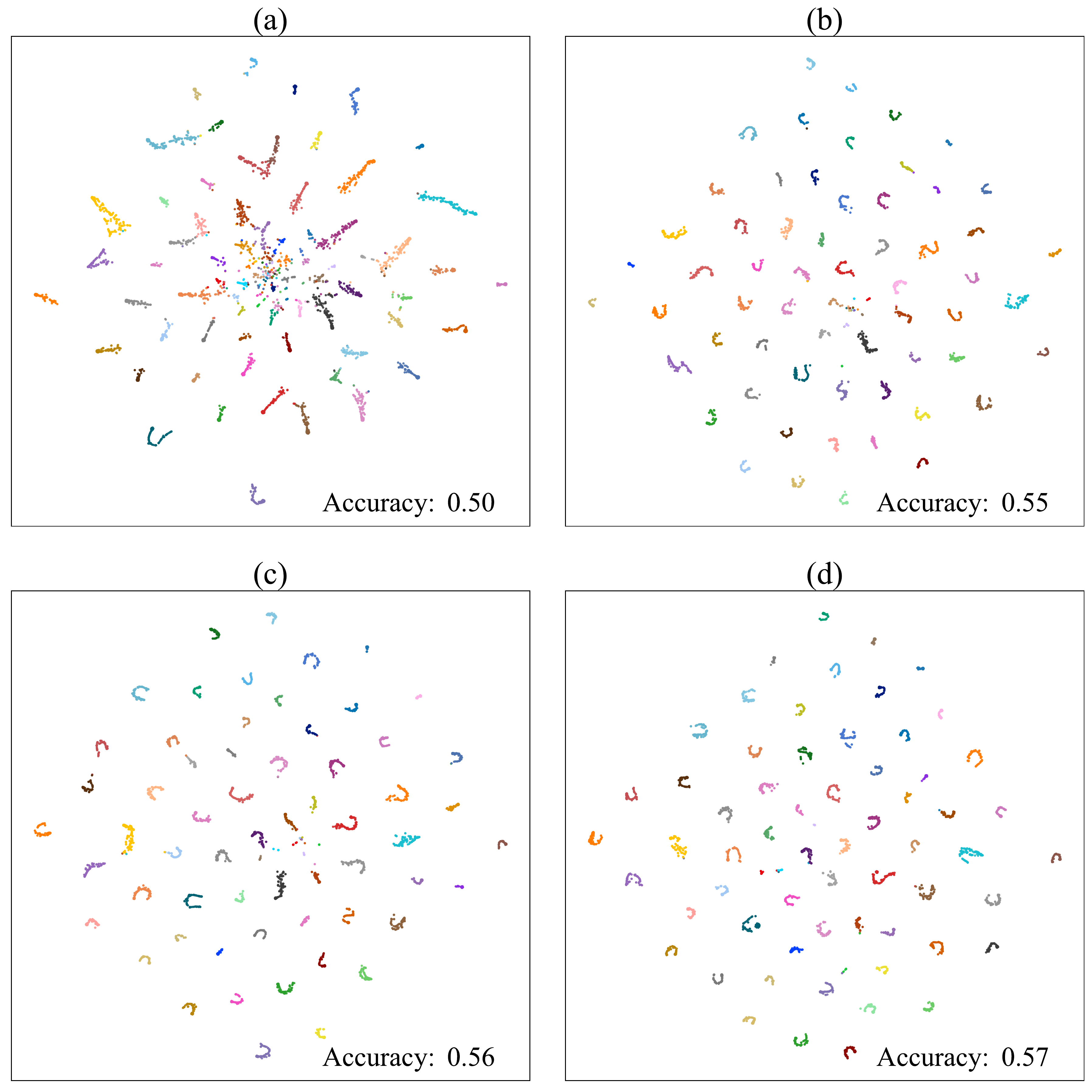}
	\caption{\emph{t}-SNE visualization of target probability prediction on task Ar$\to$Cl for (a) MaxSquare, (b) BNM, (c) CWSM, and (d) NSM. Sample points with different colors have different predicted class labels.}
	\label{fig:tsne}
\end{figure}

This section evaluates the performance of the proposed class weighted squares maximization (CWSM) and normalized squares maximization (NSM) with several baseline and state-of-the-art UDA methods on three standard benchmark datasets Office-31, Office-Home, and VisDA-2017.

\subsection{Setup}
\textbf{Office-31} \cite{saenko2010adapting} is a popular UDA benchmark dataset  which contains 4,110 images of 31 classes from three domains: A (Amazon) which contains images downloaded from the Amazon website, W (Webcam) which consists of  noisy images with low resolution taken by webcam, and D (DSLR) which contains images with high resolution captured by digital SLR camera in realworld environment. We construct 6 transfer tasks by pairwise combination of the three domains: A$\to$W, D$\to$W, W$\to$D, A$\to$D, D$\to$A, and W$\to$A. Several sample images from Office-31 are presented in Fig. \ref{fig:dataset} (a).

\textbf{Office-Home} \cite{venkateswara2017deep} is a more challenging benchmark dataset because of the larger domain gap. It contains 15,500 images of 65 classes from four domains: Ar (Art) which consists of artistic depictions, Cl (Clipart) which contains clipart images, Pr (Product) which contains images without backgrounds, and Rw (Real-World) which consists of real images taken by a camera. By pairwise combination of the four domains, there are in total 12 transfer tasks. Several sample images from Office-Home are presented in Fig. \ref{fig:dataset} (b).

\textbf{VisDA-2017}~\cite{peng2017visda} is a larger cross-domain image classification dataset in the transfer learning community, which comprises over 280K images from 12 classes. We use two domains including training domain (\textbf{Synthetic}) and validation domain (\textbf{Real}). We will evaluate 12 transfer tasks such as `train' and `truck' to evaluate our two UDA models. Several sample images from VisDA-2017 are illustrated in Fig.~\ref{fig:dataset} (c).

Following the widely used evaluation protocol, we adopt ResNet-50 as the base network and fine-tune the pre-trained model on ImageNet with new classification layer. We use the mini-batch stochastic gradient descent (SGD) with moment 0.9 to train the model. The mini-batch size $B$ is set to 36. All the experiments are implemented in PyTorch with Python programming language. For fair comparisons, we reproduce BNM \cite{cui2020towards} on both datasets in the same environment with our methods. Note that the average accuracies of BNM by our reproduction is higher than the reported results in the original paper due to our better hyper-parameter setting. We also reproduce MaxSquare \cite{chen2019domain} on the Office-Home dataset because the original study does not report results on this dataset. We denote our reproductions by \textit{method}* and  other results are duplicated from their original paper.
We fix the random seed in each run for reproducibility. For each method, we report the average accuracy of three runs with the fixed random seeds.
% to guarantee the presented results are fully reproducible.

We set $\lambda=1$ for BNM as suggested by \cite{cui2020towards}, and empirically set $\lambda=\frac{1}{C}$ for MaxSquare. For our CWSM, we set $\lambda=1$ and $r=0.5$ for three datasets. For NSM, we set $\lambda=2, r=0.5, \alpha=1, \epsilon=0$ for the Office and VisDA-2017 datasets and $\lambda=1, r=0.5, \alpha=2, \epsilon=10^{-6}$ for the Office-Home dataset. The setting of $\epsilon$ in NSM is based on the fact that $B>C$ for Office and VisDA-2017 and $B<C$ for Office-Home. A larger $\alpha$ is adopted for Office-Home because the case of $B<C$ more likely leads to small value for the term $\sum_{i,j = 1,i \ne j}^B ({\sum_{c = 1}^C {{P_{ic}}{P_{jc}}} })^r$.

\subsection{Experimental Results}
From Table \ref{tab:office}, Table \ref{tab:office_home}, and Table \ref{tab:visda}, the average accuracies of both CWSM and NSM are significantly higher than that of MaxSquare. Since CWSM and NSM differ from MaxSquare only in the equity term, this result shows the efficacy of equity maximization on UDA. Besides, it can be observed that BNM, CWSM, and NSM achieve better performance than those compared methods. This shows the effectiveness of the proposed discriminability-and-equity maximization paradigm for UDA, and BNM is also in this framework. %More importantly, we show two alternatives in more \emph{explicit} way than BNM to achieve discriminability and equity maximization.
At last, CWSM slightly outperforms BNM, while NSM outperforms both CWSM and BNM. This results from the difference between three loss functions. According to the case study in Sec. \ref{sec:case_study}, we attribute the superior performance of NSM to its robustness to outliers or noises.

We use the \emph{t}-SNE to visualize 2D representations of MaxSquare, BNM, CWSM, and NSM in Fig. \ref{fig:tsne}. By comparing with MaxSquare, CWSM and NSM produce better cluster structure: within-class compactness, between-class separability, and balanced class spread. This benefits from the incorporated equity maximization terms. It is obvious that clear cluster structure is beneficial for classification tasks. Also, NSM exhibits slightly more compact cluster structure than BNM especially in the center region. This is consistent with that the accuracy of NSM is higher than that of BNM.

\subsection{Parameter Analysis}
\label{sec:par_ana}

\begin{table}
% \small
\centering
\setlength{\tabcolsep}{0.9mm}
\caption{Accuracy (\%) of CWSM and NSM on the Office-31 dataset based on ResNet-50 with different $r$.}
\label{tab:par_office}
\begin{tabular}{cccccccc} \hline
Method&{A$\to$W}&{D$\to$W}&{W$\to$D}&{A$\to$D}&{D$\to$A}&{W$\to$A}&{Avg}\\ \hline
CWSM ($r=0.5$) &94.5 &\bf{99.0} &\bf99.7 &{\bf92.6}  &\bf75.6 &\bf{75.9}&{\bf89.6}\\
CWSM ($r=1$) &{\bf94.8} &98.6 &{99.5} &92.0  &{75.2} &75.6&89.3\\ \hline

NSM ($r=0.5$)&\bf95.4&{\bf98.8}&{\bf99.5}&\bf93.8&\bf75.6&\bf76.3&\bf89.9\\
NSM ($r=1$) &93.3&97.7&99.2&88.2&74.9&75.1&88.1\\

\hline
\end{tabular}
\end{table}

\begin{table*}[!htb]
\centering
\setlength{\tabcolsep}{0.9mm}
\caption{Accuracy (\%) of CWSM and NSM under different $r$ on the Office-Home dataset based on ResNet-50.}
\label{tab:par_office_home}
\begin{adjustwidth}{-3.8cm}{}
\begin{tabular}{cccccccccccccc}
\hline
Method          & Ar$\to$Cl & Ar$\to$Pr & Ar$\to$Rw & Cl$\to$Ar & Cl$\to$Pr & Cl$\to$Rw & Pr$\to$Ar & Pr$\to$Cl & Pr$\to$Rw & Rw$\to$Ar & Rw$\to$Cl & Rw$\to$Pr & Avg  \\
\hline
CWSM ($r=0.5$) &\bf{55.8} &\bf{75.9} &\bf{79.8}  &\bf{65.6}  &{74.6} &\bf{75.6} &\bf{63.3} &\bf{53.8} &\bf{80.7} &\bf{72.9} &57.1 &\bf{83.4} &\bf{69.9}\\
CWSM ($r=1$) &54.7 &74.3 &79.0  &64.0  &\bf{74.7} &75.2 &62.6 &52.1 &79.6 &70.6 &\bf{57.4} &82.3 &68.9\\

\hline
NSM ($r=0.5$)      &\bf56.8  &\bf76.5  &\bf80.3  & \bf66.6  &\bf 75.6  &\bf75.9  & \bf63.6  &\bf 53.9  & \bf81.0  & \bf73.3  & \bf57.9  & \bf83.8  & \bf70.4\\
NSM ($r=1$)       & 53.9  & 72.5  & 77.9  & 61.6  & 72.3  & 73.4  & 61.3  & 51.4  & 77.8  & 68.9  & 56.1  & 81.2  & 67.4 \\
\hline
\end{tabular}
\end{adjustwidth}
\end{table*}

\begin{table*}[!htb]
\centering
\caption{Accuracy (\%) of CWSM and NSM under different $r$ on the VisDA-2017 dataset based on ResNet-101.}
\label{tab:par_visda}
\begin{adjustwidth}{-3.5cm}{}
\begin{tabular}{cccccccccccccc}
\hline
Method  & plane & bcycl & bus & car & horse & knife & mcycl & person & plant & sktbrd & train & truck & Avg  \\
\hline
CWSM ($r=0.5$) &\bf93.2 &\bf{81.9} &\bf{83.2}  &{69.9}  &\bf{92.9} &\bf{80.4} &86.8 &\bf{76.8} &\bf{89.7} &\bf79.5 &\bf{88.3} &{41.2} &\bf{80.3}\\
CWSM ($r=1$) &93.1 &80.6 &80.7 &\bf71.5 &92.4 &78.9 &\bf88.2 &75.6 &89.0 & 74.0 & 87.5& \bf43.8 & 79.6\\

\hline
NSM ($r=0.5$) & 93.8  & 81.8  & \bf83.0  & \bf66.0  & \bf93.2  & \bf81.8 & \bf89.0  & \bf78.5  & \bf92.1  & \bf77.1  & \bf89.7  & 38.3  &\bf 80.4\\
NSM ($r=1$)       & \bf94.7 &\bf83.3& 77.9 &59.1& 92.4& 68.7& 84.5 &74.2& 91.5& 71.8& 88.8& \bf42.8  &77.5 \\
\hline
\end{tabular}
\end{adjustwidth}
\end{table*}

\begin{figure}[htb]
	\centering
	\includegraphics[scale=0.6]{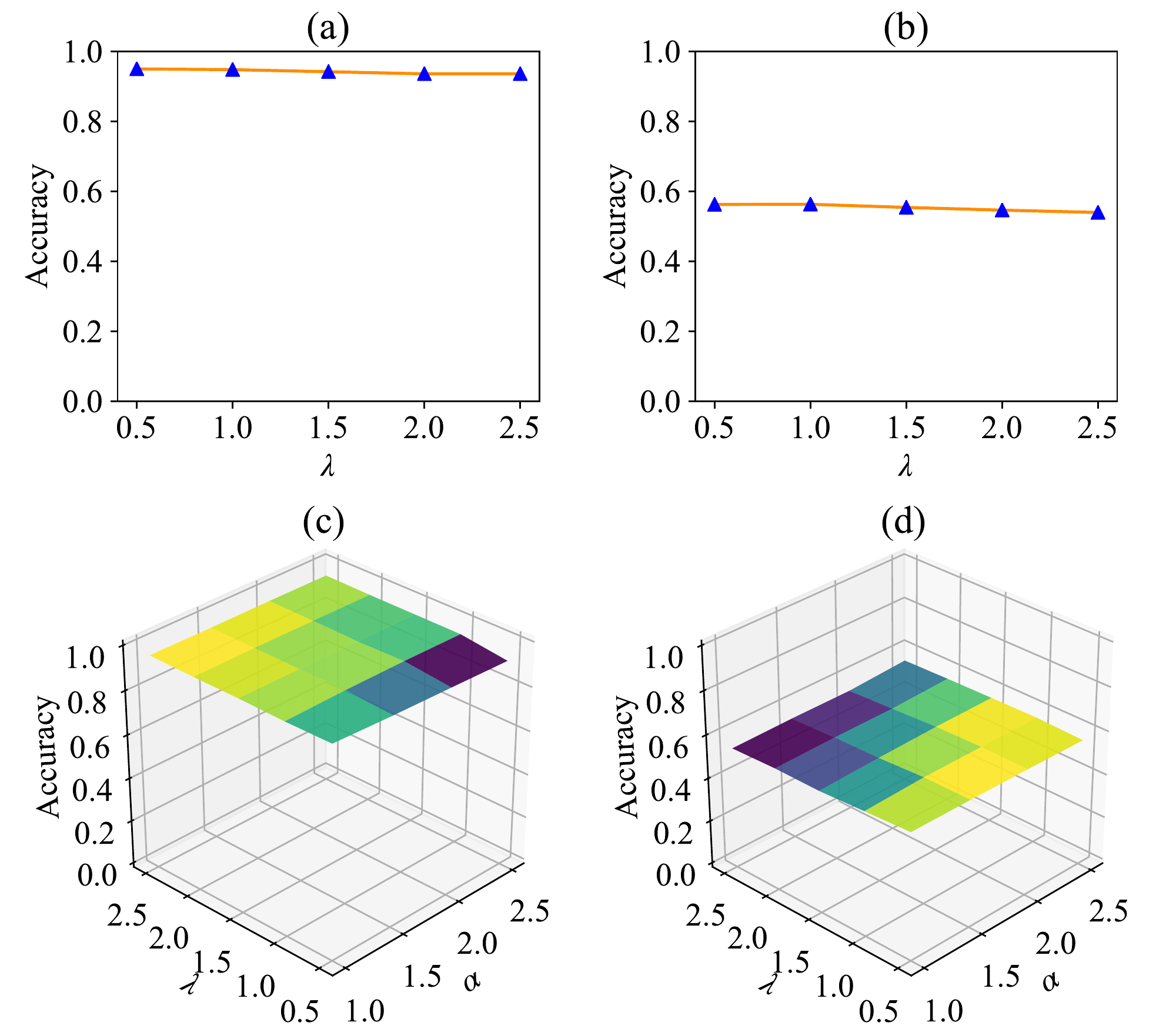}
	\caption{Accuracy of CWSM versus different $\lambda$ on (a) A$\to$W task and (b) Ar$\to$Cl task and accuracy of NSM versus different $\lambda$ and $\alpha$ on (c) A$\to$W task and (d) Ar$\to$Cl task. For both CWSM and NSM, $r=0.5$ is adopted.}
	\label{fig:par}
\end{figure}

In this section, we conduct experiments to study the influence of parameters on the proposed CWSM and NSM. CWSM and NSM have the common parameters $\lambda$ and $r$, and NSM has a private parameter $\alpha$.

For the parameter $r$, we select two typical values 0.5 and 1, and set the other parameters the same with the above-mentioned setting.
As our previous analysis, $r$ plays a role to balance discriminability and equity. For verifying this point, we define two quantities, discriminability and equity, to measure how discriminate and how balanced the target predictions are, respectively. We adopt the mean squares, i.e., $\frac{1}{N}\sum_{i=1}^N\sum_{j=1}^CP_{ij}^2$, to measure/define the discriminability. Obviously, the higher the mean squares is, the more discriminative the prediction is. Besides, we define a concretely quantitative measure for the equity as, $1-\sum_{c=1}^C \left|\frac{n_c}{N} - \frac{1}{C}\right|$, where $N$ denotes the total number of target samples, $n_c$ denotes the number of samples belonging to $c$-th class, and $C$ denotes the number of classes. The higher the equity is, the more balanced the predicted classes are. As shown in Table \ref{tab:diff}, for both CWSM and NSM, higher $r$ value ($r=1$) results in higher equity and lower discriminability. This result evidences our previous analysis.

We evaluate the performance of CWSM and NSM under $r=0.5$ and $r=1$ on both datasets. Table \ref{tab:par_office} and \ref{tab:par_office_home} show that $r=0.5$ achieves better performance than $r=1$. Combining Table \ref{tab:diff}, large $r$ results in higher equity but lower discriminability. For UDA, it is impossible to know the best proportion of the two factors in advance. Thus, the parameter $r$ needs to be tuned according to real applications. For the used two benchmark datasets, $r=0.5$ is a
good choice to achieve satisfied results.

In order to investigate the effect of $\lambda$ on CWSM and the effect of $\lambda$ and $\alpha$ on NSM, we vary $\lambda$ from the range \{0.5, 1.0, 1.5, 2.0, 2.5\} and $\alpha$ from the range \{1.0, 1.5, 2.0, 2.5\} and run CWSM and NSM on two tasks A$\to$W and Ar$\to$Cl which are from Office-31 and Office-Home datasets, respectively. As shown in Fig. \ref{fig:par}, the performance of CWSM and NSM is stable to the change of the values of $\lambda$ and $\alpha$.

\section{Conclusion}
In this paper, we reveal that nuclear norm maximization can maximize the equity in an implicit manner. Accordingly, we propose to solve unsupervised domain adaptation (UDA) with an explicit discriminability-and-equity maximization paradigm based on squares losses. In this paradigm, we offer two novel loss functions, i.e., class weighted squares maximization (CWSM) and normalized squares maximization (NSM), to simultaneously maximize discriminability and equity. We theoretically prove that the optimal solution to CWSM and NSM is with maximal discriminability and equity under mild conditions. Experimental results of two losses on three benchmarks imply the potential efficacy of the equity in UDA. %In the future, we plan to further explore other properties of both CWSM and NSM, \emph{e.g.} robustness.

\section*{Acknowledgments}
This work was supported by the National Natural Science Foundation of China [61806213, 61702134, 61906210].

\bibliography{ref}

\end{document}